\documentclass[preprint,11pt,authoryear]{elsarticle}

\usepackage{amssymb}

\usepackage[utf8]{inputenc} 
\usepackage[T1]{fontenc}    
\usepackage{hyperref}       
\usepackage{url}            
\usepackage{booktabs}       
\usepackage{amsfonts}       
\usepackage{nicefrac}       
\usepackage{microtype}      
\usepackage[dvipsnames]{xcolor}         

\usepackage{amsmath,amssymb,amsthm}
\usepackage{bm}
\usepackage{mathtools}
\usepackage[inline,shortlabels]{enumitem}
\usepackage{cleveref}
\usepackage{breakcites}
\usepackage{soul}

\usepackage{thmtools,thm-restate}

\newcommand{\cF}{\mathcal{F}}
\newcommand{\cG}{\mathcal{G}}
\newcommand{\cH}{\mathcal{H}}

\newcommand{\cP}{\mathcal{P}}

\newcommand{\cX}{\mathcal{X}}

\newcommand{\E}{\mathbb{E}}
\newcommand{\N}{\mathbb{N}}
\renewcommand{\P}{\mathbb{P}}

\newcommand{\R}{\mathbb{R}}

\newcommand{\Ind}{\mathbb{I}}
\newcommand{\pr}{\mathbb{P}}

\newcommand{\e}{\varepsilon}

\newcommand{\fat}{\mathrm{fat}}

\DeclarePairedDelimiter{\floor}{\lfloor}{\rfloor}

\DeclarePairedDelimiter{\abs}{|}{|}

\DeclarePairedDelimiter{\lr}{(}{)}
\DeclarePairedDelimiter{\lrs}{[}{]}
\DeclarePairedDelimiter{\lrc}{\{}{\}}

\DeclareMathOperator{\VC}{VCdim}

\newcommand{\lrb}[1]{\left(#1\right)}
\newcommand{\brb}[1]{\bigl(#1\bigr)}
\newcommand{\Brb}[1]{\Bigl(#1\Bigr)}
\newcommand{\bbrb}[1]{\biggl(#1\biggr)}
\newcommand{\Bbrb}[1]{\Biggl(#1\Biggr)}
\newcommand{\lsb}[1]{\left[#1\right]}

\newcommand{\lcb}[1]{\left\{#1\right\}}
\newcommand{\bcb}[1]{\bigl\{#1\bigr\}}

\newcommand{\labs}[1]{\left\lvert#1\right\rvert}
\newcommand{\babs}[1]{\bigl\lvert#1\bigr\rvert}

\newcommand\restr[2]{{
  \left.\kern-\nulldelimiterspace 
  #1 
  \littletaller 
  \right|_{#2} 
  }}

\newcommand{\littletaller}{\mathchoice{\vphantom{\big|}}{}{}{}}

\renewcommand{\tilde}{\widetilde}

\theoremstyle{definition}
\newtheorem*{examplex*}{Example}

\newenvironment{example*}
  {\pushQED{\qed}\begin{examplex*}}
  {\popQED\end{examplex*}}

\theoremstyle{plain}

\newtheorem{theorem}{Theorem}

\newtheorem*{remark*}{Remark}

\newtheorem{definition}{Definition}

\newtheorem*{observation*}{Observation}

\newlist{refenumerate}{enumerate}{1} 
\setlist[refenumerate]{label={({\alph*})},ref=({\alph*})}
\crefname{refenumeratei}{item}{items}
\Crefname{refenumeratei}{Item}{Items}

\journal{Information Processing Letters}

\begin{document}

\begin{frontmatter}

\title{An Improved Uniform Convergence Bound with Fat-Shattering Dimension}

\author[inst1,inst2]{Roberto Colomboni}
\author[inst1,inst2]{Emmanuel Esposito}
\author[inst1]{Andrea Paudice}

\affiliation[inst1]{
organization={Department of Computer Science, Università degli Studi di Milano},
addressline={\\Via Giovanni Celoria, 18}, 
city={Milan},
postcode={20131}, 
country={Italy}}

\affiliation[inst2]{
organization={Computational Statistics and Machine Learning, Istituto Italiano di Tecnologia},
addressline={\\Via Enrico Melen, 83}, 
city={Genoa},
postcode={16152}, 
country={Italy}}

\begin{abstract}
The fat-shattering dimension characterizes the uniform convergence property of real-valued functions.
The state-of-the-art upper bounds feature a multiplicative squared logarithmic factor on the sample complexity, leaving an open gap with the existing lower bound.
We provide an improved uniform convergence bound that closes this gap.
\end{abstract}

\begin{keyword}
Uniform convergence \sep Fat-shattering dimension
\end{keyword}

\end{frontmatter}

\section{Introduction}
Given a class of real-valued functions $\cF$ with domain $\cX$, it is said that $\cF$ enjoys the uniform convergence property if, for every $\cX$-valued i.i.d.\ process $X,X_1,X_2,\dots$, the sequence of empirical means $\frac{1}{m} \sum_{i=1}^m f(X_i)$ converges in probability to its expectation $\E\lsb{f(X)}$, uniformly over $f \in \cF$. Formally, $\cF$ enjoys the uniform convergence property if, for every $\e,\delta>0$, there exists $\hat{m} \coloneqq \hat{m}(\e,\delta) \in \N$ such that, for every $m \ge \hat{m}$ and every $\cX$-valued i.i.d.\ process $X,X_1,X_2,\dots$, it holds
\[
\P \lrb{  \sup_{f \in \cF} \labs{\frac{1}{m} \sum_{i=1}^m f(X_i) - \E\lsb{f(X)}} > \e } \le \delta \;.
\]
Among all the (functions) $\hat{m}$ for which the previous property holds, the smallest of them (pointwise in $\e,\delta$), namely $m^*$, is called the sample complexity for the uniform convergence of $\cF$.

Uniform convergence is a fundamental tool in learning theory. Indeed, we can learn any class of functions $\cF$ that enjoys uniform convergence via empirical risk minimization.
Besides learnability, uniform convergence has notable practical applications. In particular, whenever $\cF$ enjoys uniform convergence, we can estimate the risk of any model in $\cF$ computing its empirical risk over the same dataset used to select the model, an aspect that can be especially useful when the model is selected via (heuristic) approximations of algorithms featuring theoretical guarantees. In all these applications it is crucial to have sharp estimates of the sample complexity~$m^*$.

A large body of work has focused on identifying conditions implying uniform convergence \citep{Vapnik1968,Vapnik1971,Pollard1986,BenDavid1992,Alon1997,Bartlett1998}.
In particular, \cite{Alon1997} showed that the fat-shattering dimension (introduced by \citet{Kearns1994}) characterizes the uniform convergence property for real-valued functions. However, state-of-the-art estimates on the sample complexity $m^*$ for the uniform convergence \citep{Bartlett1995} have an accuracy gap of order $\ln^2(1/\e)$ when compared to the corresponding lower bound for the special case of binary functions \citep{Vapnik1968}.
These bounds are extensively used in the current literature (see, e.g., \cite{Attias2022a,Attias2022b,Belkin2018,Hu2022}) and, regrettably, have not been improved ever since.

In this work, we close this gap removing the exogenous term $\ln^2(1/\e)$.
Our improvement builds upon a carefully designed chaining argument leveraging sharp estimates (up to constants) for the metric entropy based on the fat-shattering dimension~\citep{Rudelson2006}.

\section{Preliminaries}
Throughout the paper we use the following notation. If $A$ is any finite set, the number of elements in $A$ is denoted by $|A|$. $\N$ is the set of natural numbers $\{1,2,\dots\}$,  $\R$ is the set of real numbers. 
For each $m \in \N$, we denote by $[m]$ the set $\lcb{1,\dots,m}$, and we freely identify $m$-uples of real numbers with real-valued functions having $[m]$ as domain (i.e., $\R^m \cong \R^{[m]}$ via $x \mapsto (i \mapsto x_i)$).
For each $m \in \N$ and each $p \in \{1,2\}$, we define $d_p \colon \R^m \times \R^m \to [0,\infty)$ as the metric defined by $d_p(g,h) \coloneqq \lrb{\frac{1}{m}\sum_{i=1}^m |g(i)-h(i)|^p}^{1/p}$, for any $g,h \in \R^m$.
If $(\cX,d)$ is a metric space, $\e > 0$ and $x \in \cX$, the closed ball of radius $\e$ centered at $x$ is denoted by $B_\e(x)$. In this case, for any $\e > 0$ and any $\tilde{\cX} \subset \cX$, we recall that $\tilde{\cX}$ is said to be an $\e$-net if $\cX \subset \bigcup_{x \in \tilde{\cX}} B_\e(x)$, while $\tilde{\cX}$ is said to be an $\e$-separated set if $d(x_1,x_2)> \e$ for any two distinct points $x_1,x_2 \in \tilde{\cX}$. The $\e$-packing number $\cP\brb{\cX,d,\e}$ of the metric space $(\cX,d)$ is the maximum number of elements of any $\e$-separated set, whenever this maximum exists; otherwise we set it to $\infty$.
For any $n \in \N$, if $A_1,\dots,A_n$ are non-empty subsets of some vector space $V$, we denote their Minkowski sum using the notation $A_1 + \dots + A_n \coloneqq \{ v_1+\dots+v_n \mid \forall i \in [n], v_i \in A_i \}$.
We recall that a Rademacher random variable (with respect to some underlying probability measure $\P$) is any random variable $Z$ such that $\P(Z = 1) = 1/2 = \P(Z=-1)$.

\section{The Uniform Convergence Bound}
In this section we present our result, which improves on state-of-the-art bounds based on the fat-shattering dimension, together with a proof. 
We start with the relevant definitions.

\begin{definition}[Fat-shattering Dimension]
Let $\cF \subset \R^{\cX}$ and $\gamma > 0$. Let $m \in \N$ and $S=\{x_1,\dots,x_m\} \subset \cX$. We say that $S$ is $\gamma$-shattered by $\cF$ if there exists a function $r\colon S \rightarrow \R$ such that, for every $B \subset S$, there exists $f_B \in \cF$ satisfying
\begin{align*}
\forall x \in B, &\qquad f_B(x) \ge r(x) + \gamma \;,\\
\forall x \in S\setminus B, &\qquad f_B(x) \le r(x) - \gamma \;.
\end{align*}
We define the fat-shattering dimension $\fat_{\gamma}(\cF)$ as the maximum number of elements of a set $S$ that is $\gamma$-shattered by $\cF$, when this maximum exists; otherwise, we set $\fat_{\gamma}(\cF) = \infty$.
\end{definition}

The fat-shattering dimension is a scale-sensitive generalization to real-valued functions of the classical Vapnik-Chervonenkis dimension for Boolean functions \citep{Vapnik1968}. 
It is well known \citep{Alon1997} that the finiteness of the fat-shattering dimension for a class of functions $\cF$ characterizes the uniform convergence of $\cF$.

We are now ready to state our main theorem.

\begin{theorem}\label{thm:uniform-convergence}
There exist universal constants $C,c > 0$ such that the following holds.
For any $a<b$, any $\cF \subseteq \R^{[a,b]}$,\footnote{To avoid measurability pathologies (see \citet{Ben-David2015}), and for the sake of simpliticity, we carry out the proof under the further assumption that the class $\cF$ is countable. This assumption can be greatly relaxed \citep{Alon1997} relying on measurability conditions such as the ‘‘image admissible Suslin’’ property \citep[Section~10.3.1, page~101]{Dudley1984}.} and any probability measure $\P$, if $X,X_1,X_2,\dots$ is a $\P$-i.i.d.\ $\cX$-valued sequence of random variables,
then, for every $\e>0$ satisfying $\fat_{c\e}(\cF) < \infty$, every $\delta \in (0,1)$, and every $m \in \N$ satisfying
\begin{equation} \label{eq:uniform-convergence-sample-complexity}
m \ge
C \cdot \frac{(b-a)^2}{\e^2} \lrb{\fat_{c \e}(\cF) + \ln\frac{1}{\delta}} \enspace,
\end{equation}
we have that, with probability at least $1-\delta$,
\[
\sup_{f \in \cF} \labs{\frac{1}{m} \sum_{i=1}^m f(X_i) - \E\lsb{f(X)}} \le \e \enspace.
\]
\end{theorem}

Before presenting a proof of \Cref{thm:uniform-convergence}, some remarks are in order.

The best previously known bound on the sample complexity of the uniform convergence of $[a,b]$-valued functions was of the order of
\[
\frac{(b-a)^2}{\e^2} \Bbrb{\fat_{\e/5}(\cF)\ln^2\bbrb{\frac{b-a}{\e}}+\ln \frac{1}{\delta}}
\]
(see Theorem 9, Eq.\ (5) in \cite{Bartlett1995}\footnote{The original bound was stated for $[0,1]$-valued functions, but with a straightforward adaptation of the proof, it can be extended to $[a,b]$-valued functions while preserving the scale of the fat-shattering dimension.}).
The bound of \Cref{eq:uniform-convergence-sample-complexity} improves on it by removing the extra $\ln^2\brb{(b-a)/\e}$ factor. Our bound is optimal (up to the constants $c$ and $C$, which we did not try to optimize or estimate) as for the dependence on $\e$ and $\delta$. Indeed, if $\cF \subset \{0,1\}^{\cX}$ and $\e < 1/(2c)$, then $\fat_{c\e}(\cF) = \VC(\cF)$. In this case, it is well-known that to ensure uniform convergence, at least  order of $\e^{-2}(\VC(\cF)+\ln(1/\delta))$ samples are required \citep{Vapnik1968}.

Our proof does not rely on discretizing the codomain $[a,b]$ as was done in previous work \citep{Bartlett1995,Alon1997}. We avoid this use of discretization relying on \Cref{lem:net} and the breakthrough result of \citet{Rudelson2006}, which bounds directly the packing number of certain metric spaces of functions in terms of a fat-shattering dependent quantity.

\section{Auxiliary results} \label{sec:uniform-convergence-sketch}
The proof follows the pattern of chaining techniques \citep{Talagrand1994} and, for the sake of clarity, it is provided with the aid of a sequence of lemmas.

Fix $\cF, \P, X,X_1,X_2,\dots, a,b, \e, \delta$ as in the statement of the theorem. Also, fix $m \in \N$.
The first tool is a symmetrization lemma, which can be proved along the lines of the corresponding symmetrization lemma for $[0,1]$-valued functions (and that can be found, e.g., in \citet[Lemma~10]{Bartlett1995}).

\begin{restatable}[Symmetrization]{relem}{symmetrizationlemma} \label{lem:symmetrization}
    If $m \ge 4\ln(2)\cdot \brb{(b-a)/\e}^2$, then
    \begin{align*}
        &\P\lr*{\sup_{f \in \cF} \abs*{\frac1m \sum_{i=1}^m f(X_i) - \E\lrs*{f(X)}} > \e} \\
        &\quad\le 2\cdot \P\lr*{\sup_{f \in \cF} \abs*{\frac1m \sum_{i=1}^m \brb{f(X_i) - f(X_{m+i})}} > \frac{\e}{2}} \;.
    \end{align*}
\end{restatable}

The second tool we need is a permutation lemma, which can be proved following the lines of the corresponding permutation lemma that can be found, e.g., in \citet[Lemma~4.5]{Anthony1999}.

\begin{restatable}[Permutation lemma]{relem}{permutationlemma} \label{lem:permutation}
    Let $Z_1, \dots, Z_m$ be a family of $\P$-independent Rademacher random variables.
    Then,
    \begin{align*}
        &\P\lr*{\sup_{f \in \cF} \abs*{\frac1m \sum_{i=1}^m \brb{f(X_i) - f(X_{m+i})}} > \frac{\e}{2}} \\
        &\le \sup_{(x_1, \dots, x_{2m}) \in \cX^{2m}} \P\lr*{\sup_{f \in \cF} \abs*{\frac1m \sum_{i=1}^m Z_i\brb{f(x_i) - f(x_{m+i})}} > \frac{\e}{2}} \;.
    \end{align*}
\end{restatable}

In the light of the previous two lemmas, it will be sufficient to estimate the last probability involving the supremum of linear combinations of Rademacher random variables.

From now on, we fix a family $Z_1,\dots,Z_m$ of $\P$-i.i.d.\ Rademacher random variables.

For each $\mathbf{x} \coloneqq (x_1, \dots, x_{2m}) \in \cX^{2m}$, define the family
\[
    \cF(\mathbf{x}) \coloneqq \lrc{g \in \R^{2m} \mid \exists f \in \cF, \forall i \in [2m], g(i) = f(x_i)}
\]
of vectors in $\R^{2m}$ that represent the restrictions of the functions in $\cF$ to the sample $\mathbf{x}$.

For each $\mathbf{x} \in \cX^{2m}$, we fix an $\e/8$-separated $\e/8$-net $\cF_{\e}(\mathbf{x})$ of the metric space $\brb{\cF(\mathbf{x}),d_2}$. These sets can be built following an iterative procedure where, at each step, we add another element whose distance from any already selected element is greater than $\e/8$. This procedure terminates after at most $\cP\brb{\cF(\mathbf{x}),d_2,\e/8}$ steps, and we note explicitly that $\cP\brb{\cF(\mathbf{x}),d_2,\e/8}<\infty$ as a consequence of \Cref{lem:Vershynin}. When this procedure stops, every element of $\cF(\mathbf{x})$ is within $\e/8$ distance from some element in $\cF_{\e}(\mathbf{x})$.

The next ingredient is a lemma whose purpose is to reduce the problem of bounding the supremum over the whole family $\cF$ to another problem where the supremum is taken with respect to the $\e/8$-separated set $\cF_\e(\mathbf{x})$, over which we plan to implement a chaining procedure.
We explicitly note that, to prove the following lemma, the sole property of $\cF_\e(\mathbf{x})$ we use is that it is an $\e/8$-net of the metric space $(\cF(\mathbf{x}), d_2)$.
\begin{restatable}{relem}{netlemma} \label{lem:net}
\begin{align*}
&\sup_{(x_1, \dots, x_{2m}) \in \cX^{2m}} \P\lr*{\sup_{f \in \cF} \abs*{\frac1m \sum_{i=1}^m Z_i\brb{f(x_i) - f(x_{m+i})}} > \frac{\e}{2}} \\
&\qquad\le \sup_{\mathbf{x} \in \cX^{2m}} \P\lrb{\sup_{f \in \cF_{\e}(\mathbf{x})} \abs*{\frac1m \sum_{i=1}^m Z_i\brb{f(i) - f(m+i)}} > \frac{\e}{4}} \;.
\end{align*}
\end{restatable}
\begin{proof}
Recall that, for each $\mathbf{x} \in \cX^{2m}$, the set $\cF_{\e}(\mathbf{x})$ is an $\e/8$-net of the metric space $\brb{\cF(\mathbf{x}),d_2}$.
For each $\mathbf{x} \coloneqq (x_1, \dots, x_{2m}) \in \cX^{2m}$ and each $\xi \coloneqq (\xi_1,\dots,\xi_m) \in \{-1,1\}^m$, select $f_{\mathbf{x},\xi} \in \cF(\mathbf{x})$ such that
\[
    \abs*{\frac1m \sum_{i=1}^m \xi_i\brb{f_{\mathbf{x},\xi}(i) - f_{\mathbf{x},\xi}(m+i)}} > \e/2
\]
whenever it is possible, otherwise select $f_{\mathbf{x},\xi} \in \cF(\mathbf{x})$ arbitrarily. Notice that, if it holds that
\[
    \sup_{f \in \cF} \abs*{\frac1m \sum_{i=1}^m \xi_i\brb{f(x_{i}) - f(x_{m+i})}} > \e/2
\]
then also
\[
    \abs*{\frac1m \sum_{i=1}^m \xi_i\brb{f_{\mathbf{x},\xi}(i) - f_{\mathbf{x},\xi}(m+i)}} > \e/2
\]
holds, and vice versa.

For each $\mathbf{x} \in \cX^{2m}$ and each $\xi \in \{-1,1\}^m$, let $f_{\mathbf{x},\xi, \e} \in \cF_{\e}(\mathbf{x})$ be such that
\[
    d_2(f_{\mathbf{x},\xi},f_{\mathbf{x},\xi, \e} ) \le \frac{\e}{8} \;.
\]
Then, for each $\mathbf{x} \coloneqq (x_1,\dots,x_{2m}) \in \cX^{2m}$ we have
{\allowdisplaybreaks
\begin{align*}
&\pr\lrb{\sup_{f \in \cF} \abs*{\frac1m \sum_{i=1}^m Z_i\brb{f(x_{i}) - f(x_{m+i})}} > \frac{\e}2} \\
&=
\frac1{2^m} \sum_{\xi \in \{-1,1\}^{m}} \Ind\lrc*{\sup_{f \in \cF}\abs*{\frac1m \sum_{i=1}^m \xi_i\brb{f(x_i) - f(x_{m+i})}} > \frac{\e}2} \\
& =
\frac1{2^m} \sum_{\xi \in \{-1,1\}^{m}} \Ind\lrc*{\abs*{\frac1m \sum_{i=1}^m \xi_i\brb{f_{\mathbf{x},\xi}(i) - f_{\mathbf{x},\xi}(m+i)}} > \frac{\e}2}
\\
&\le
\frac1{2^m} \sum_{\xi \in \{-1,1\}^{m}} \Ind\lrc*{\abs*{\frac1m \sum_{i=1}^m \xi_i\brb{f_{\mathbf{x},\xi,\e}(i) - f_{\mathbf{x},\xi,\e}(m+i)}} > \frac{\e}4}
\\
&\quad+
\frac1{2^m} \sum_{\xi \in \{-1,1\}^{m}} \Ind\lrc*{ d_1(f_{\mathbf{x},\xi},f_{\mathbf{x},\xi,\e}) > \frac{\e}8}
\\
&\le
\frac1{2^m} \sum_{\xi \in \{-1,1\}^{m}} \Ind\lrc*{\abs*{\frac1m \sum_{i=1}^m \xi_i\brb{f_{\mathbf{x},\xi,\e}(i) - f_{\mathbf{x},\xi,\e}(m+i)}} > \frac{\e}4}
\\
&\quad+
\frac1{2^m} \sum_{\xi \in \{-1,1\}^{m}} \Ind\lrc*{ d_2(f_{\mathbf{x},\xi},f_{\mathbf{x},\xi,\e}) > \frac{\e}8}
\\
&=
\frac1{2^m} \sum_{\xi \in \{-1,1\}^{m}} \Ind\lrc*{\abs*{\frac1m \sum_{i=1}^m \xi_i\brb{f_{\mathbf{x},\xi,\e}(i) - f_{\mathbf{x},\xi,\e}(m+i)}} > \frac{\e}4}
\\
&\le
\frac1{2^m} \sum_{\xi \in \{-1,1\}^{m}} \Ind\lrc*{\sup_{f \in \cF_{\e}(\mathbf{x})} \abs*{\frac1m \sum_{i=1}^m \xi_i\brb{f(i) - f(m+i)}} > \frac{\e}4 }
\\
&=
\pr\lrb{\sup_{f \in \cF_{\e}(\mathbf{x})} \abs*{\frac1m \sum_{i=1}^m Z_i\brb{f(i) - f(m+i)}} > \frac{\e}4 }
\;. \qedhere
\end{align*}
}
\end{proof}

Leveraging Hoeffding's inequality \citep{Hoeffding1963}, we can prove the following lemma, which can be viewed as a multiscale concentration inequality.

\begin{restatable}{relem}{multiscalelemma} \label{lem:multiscale}
Let $\tilde{l} \in \N$.
Consider $\e_0, \dots, \e_{\tilde{l}} > 0$ such that $\sum_{j=0}^{\tilde{l}} \e_j \le \e/4$.
For each $\mathbf{x} \in \cX^{2m}$, let $\tilde{\cH}_0(\mathbf{x}), \dots, \tilde{\cH}_{\tilde{l}}(\mathbf{x}) \subset \cF_\e(\mathbf{x})$ such that $\cF_\e(\mathbf{x}) \subset \tilde{\cH}_0(\mathbf{x}) + \dots + \tilde{\cH}_{\tilde{l}}(\mathbf{x})$.
Then,
\begin{align*}
\sup_{\mathbf{x} \in \cX^{2m}} &\P\lr*{\sup_{f \in \cF_\e(\mathbf{x})} \abs*{\frac1m \sum_{i=1}^m Z_i\brb{f(i) - f(m+i)}} > \frac{\e}{4}} \\
&\le 2\sum_{j=0}^{\tilde{l}} \sup_{\mathbf{x} \in \cX^{2m}} \sum_{h \in \tilde{\cH}_{j}(\mathbf{x})} \exp\lr*{-\frac12 \cdot \frac{\e_j^2 m^2}{\sum_{i=1}^m \brb{h(i) - h(m+i)}^2}} \;.
\end{align*}
\end{restatable}
\begin{proof}
Fix $\mathbf{x} \coloneqq (x_1,\dots,x_{2m}) \in \cX^{2m}$. For each $f \in \cF_\e(\mathbf{x})$, since $\cF_\e(\mathbf{x}) \subset \tilde{\cH}_0(\mathbf{x}) + \dots + \tilde{\cH}_{\tilde{l}}(\mathbf{x})$, we can (and do) select $h_{0}^f \in \tilde{\cH}_0(\mathbf{x}), \dots, h_{\tilde{l}}^f \in \tilde{\cH}_{\tilde{l}}(\mathbf{x})$ such that $f = h_{0}^f + \dots + h_{\tilde{l}}^f$. Furthermore, notice that for each $(\xi_1,\dots,\xi_m) \in \{-1,1\}^m$ it holds
\begin{align*}
    &\lcb{\sup_{f \in \cF_\e(\mathbf{x})} \abs*{\frac1m \sum_{i=1}^m \xi_i\brb{f(i) - f(m+i)}} > \frac{\e}{4}}
\\
&\quad\subset
    \bigcup_{j=0}^{\tilde{l}}\lcb{\sup_{f \in \cF_\e(\mathbf{x})} \abs*{\frac1m \sum_{i=1}^m \xi_i\brb{h_{j}^f(i) - h_{j}^f(m+i)}} > \e_j}\;.
\end{align*}
It follows that
{
\allowdisplaybreaks
\begin{align*}
&\P\lr*{\sup_{f \in \cF_\e(\mathbf{x})} \abs*{\frac1m \sum_{i=1}^m Z_i\brb{f(i) - f(m+i)}} > \frac{\e}{4}}
\\
&= 
\frac1{2^m} \sum_{\xi \in \{-1,1\}^m} \Ind\lrc*{\sup_{f \in \cF_\e(\mathbf{x})} \abs*{\frac1m \sum_{i=1}^m \xi_i\brb{f(i) - f(m+i)}} > \frac{\e}4} \\
&\le \sum_{j=0}^{\tilde{l}} \frac1{2^m} \sum_{\xi \in \{-1,1\}^m} \Ind\lrc*{\sup_{f \in \cF_\e(\mathbf{x})} \abs*{\frac1m \sum_{i=1}^m \xi_i\brb{h_{j}^f(i) - h_{j}^f(m+i)}} > \e_j} \\
&\le \sum_{j=0}^{\tilde{l}} \frac1{2^m} \sum_{\xi \in \{-1,1\}^m} \Ind\lrc*{\sup_{h \in \tilde{\cH}_{j}(\mathbf{x})} \abs*{\frac1m \sum_{i=1}^m \xi_i\brb{h(i) - h(m+i)}} > \e_j} \\
&\le \sum_{j=0}^{\tilde{l}}  \sum_{h \in \tilde{\cH}_{j}(\mathbf{x})} \frac1{2^m} \sum_{\xi \in \{-1,1\}^m} \Ind\lrc*{\abs*{\frac1m \sum_{i=1}^m \xi_i\brb{h(i) - h(m+i)}} > \e_j} \\
&= \sum_{j=0}^{\tilde{l}}  \sum_{h \in \tilde{\cH}_{j}(\mathbf{x})} \P\lr*{\abs*{\frac1m \sum_{i=1}^m Z_i\brb{h(i) - h(m+i)}} > \e_j} \eqqcolon (\star) \enspace.
\end{align*}
}
Now, for each $j \in \{0,1,\dots,\tilde{l}\}$, and each $h \in \tilde{\cH}_{j}(\mathbf{x})$ the sequence
\[
    \lrb{W_i^h(\mathbf{x}) := Z_i\brb{h(i) - h(m+i)}}_{i \in [m]}
\]
is a sequence of bounded zero-mean independent random variables.
More precisely, we notice that
\[
    -\abs{h(i)-h(m+i)} \le W_i^h(\mathbf{x}) \le \abs{h(i)-h(m+i)} \enspace.
\]
Leveraging Hoeffding's inequality \citep{Hoeffding1963}, we obtain
\begin{align*}
(\star)
&\le
2\sum_{j=0}^{\tilde{l}} \sum_{h \in \tilde{\cH}_{j}(\mathbf{x})} \exp\lr*{-\frac12 \cdot \frac{\e_j^2 m^2}{\sum_{i=1}^m \brb{h(i) - h(m+i)}^2}} \enspace.
\end{align*}
Taking the supremum over $\mathbf{x} \in \cX^{2m}$ on the first and the last term of this chain of inequalities, and switching the supremum with the sum over $j \in \{0,\dots,\tilde{l}\}$ on the last expression, we obtain the conclusion. \qedhere
\end{proof}

Now, for each $\mathbf{x} \in \cX^{2m}$, we need to build a suitable sequence $\tilde{\cH}_0(\mathbf{x}), \dots, \tilde{\cH}_{\tilde{l}}(\mathbf{x})$ to which we want to apply \Cref{lem:multiscale}.
Our choice for such a sequence follows a chaining argument \citep{Talagrand1994}.

From now on, we fix $l \coloneqq \floor{\log_2\brb{(b-a)/\e}} + 4$ and, for each $\mathbf{x} \in \cX^{2m}$, we define by induction on $j = 0,1,\dots, l$, the sets $\cG_0(\mathbf{x}), \dots, \cG_{l}(\mathbf{x}) \subset \cF_\e(\mathbf{x})$ in the following way:
\begin{itemize}
    \item $\cG_0(\mathbf{x}) \coloneqq \lrc{g_0}$, for an arbitrary choice of $g_0 \in \cF_\e(\mathbf{x})$.
    \item For any $j \in [l]$, we initially define $\cG_j(\mathbf{x}) \coloneqq \cG_{j-1}(\mathbf{x})$. Then, iteratively, we add elements $f \in \cF_{\e}(\mathbf{x})$ to $\cG_j(\mathbf{x})$ for which $d_2(f,g) > (b-a) \cdot 2^{-j}$ for every other element $g$ already in $\cG_j(\mathbf{x})$. The procedure is carried out until we can no longer add other elements. \footnote{Note that this process has to come to an end, since $|\cF_\e(\mathbf{x})|\le \cP\brb{\cF(\mathbf{x}),d_2,\e/8}<\infty$, as a consequence of \Cref{lem:Vershynin}.}
\end{itemize} 
Notice that, by construction, for each $\mathbf{x} \in \cX^{2m}$ and each $j \in \{0,1,\dots,l\}$, the set $\cG_j(\mathbf{x})$ is a $(b-a)\cdot 2^{-j}$-net and $(b-a)\cdot 2^{-j}$-separated set of $\brb{\cF_{\e}(\mathbf{x}),d_2}$, which implies that, for any $f \in \cF_\e(\mathbf{x})$, there exists ---and hence we can (and do) select--- an element $\pi_j(f) \in \cG_j(\mathbf{x})$ such that $d_2\brb{f,\pi_j(f)} \le (b-a) \cdot 2^{-j}$.
For each $\mathbf{x} \in \cX^{2m}$ we define
\begin{align}
    \cH_0(\mathbf{x}) &\coloneqq \cG_0(\mathbf{x}) \nonumber\\
    \cH_j(\mathbf{x}) &\coloneqq \bcb{ g - \pi_{j-1}(g) \mid g \in \cG_j(\mathbf{x}) }, \forall j \in [l]\;.
    \label{eq:chaining}
\end{align}

The relevant properties of this sequence of sets are summarized by the following lemma.

\begin{restatable}{relem}{chaininglemma} \label{lem:chaining}
    For any $\mathbf{x} \in \cX^{2m}$, consider $\cH_0(\mathbf{x}), \dots, \cH_{l}(\mathbf{x})$ defined as in~\eqref{eq:chaining}. It holds that
    \begin{enumerate}
        \item $\cF_\e(\mathbf{x}) \subset \cH_0(\mathbf{x}) + \dots + \cH_{l}(\mathbf{x})$\;.
        \item $\forall j \in \lrc{0, \dots, l}, \forall h \in \cH_j(\mathbf{x}), \sum_{i=1}^m \brb{h(i) - h(m+i)}^2
        \le 16m (b-a)^2 4^{-j}$\;.
        \item $\forall j \in \lrc{0, \dots, l}, \babs{\cH_j(\mathbf{x})} \le \cP\brb{\cF_\e(\mathbf{x}),d_2,(b-a) \cdot 2^{-j}}$\;.
    \end{enumerate}
\end{restatable}
\begin{proof}
    Fix an arbitrary $\mathbf{x} \in \cX^{2m}$.
    First, we prove that $\cF_\e(\mathbf{x}) = \cG_{l}(\mathbf{x})$.
    We know that $\cG_{l}(\mathbf{x}) \subset \cF_\e(\mathbf{x})$ by construction.
    Consider now any $f \in \cF_\e(\mathbf{x})$. By our choice of $l$, we have that
    \begin{align*}
        d_2(f, \pi_l(f)) \le (b-a)\cdot 2^{-l} \le \frac{\e}{8} \enspace.
    \end{align*}
    If $f \neq \pi_l(f)$ were true then we would have that $d_2(f, \pi_l(f)) > \e/8$ because $\cF_{\e}(\mathbf{x})$ is an $\e/8$-separated set, which is a contradiction.
    Then, it holds that $f=\pi_l(f)$ and thus $\cF_{\e}(\mathbf{x}) \subset \cG_{l}(\mathbf{x})$.
    
    Second, for each $j \in \{0, \dots, l\}$, and each $f \in \cG_{j}(\mathbf{x})$, we prove that there exist $h_0 \in \cH_0(\mathbf{x}), \dots, h_j \in \cH_j(\mathbf{x})$ such that
    $f = \sum_{k=0}^j h_k$.
    We prove this claim by induction on $j = 0,1,\dots,l$. The base case $j=0$ is trivial. Assuming the claim holds for $j$ (with $j<l$), if $f \in \cG_{j+1}(\mathbf{x})$, we have that $\pi_j(f) = \sum_{i=0}^j h_i$ for some $h_0 \in \cH_0(\mathbf{x}), \dots, h_j \in \cH_j(\mathbf{x})$ by the inductive hypothesis, and thus $f = \brb{f-\pi_j(f)} + \pi_j(f) = \sum_{i=0}^{j+1} h_i$ for $h_{j+1} \coloneqq f-\pi_j(f) \in \cH_{j+1}(\mathbf{x})$, hence proving the claim.
    
    The above property for the specific case of $j=l$ implies that $\cF_\e(\mathbf{x}) = \cG_{l}(\mathbf{x}) \subset \cH_0(\mathbf{x}) + \dots + \cH_{l}(\mathbf{x})$.
    This shows that the first point in the statement holds.

    Consider now any $h \in \cH_j(\mathbf{x})$ for each $j \in [l]$.
    By definition of $h$, there exists some $g \in \cG_j(\mathbf{x})$ such that $h = g-\pi_{j-1}(g)$.
    Then,
    \begin{align*}
        \sum_{i=1}^m &\brb{h(i) - h(m+i)}^2
        \le 4m\cdot d_2(g, \pi_{j-1}(g))^2 \\
        &\le 4m\cdot (b-a)^2 2^{-2(j-1)}
        = 16m (b-a)^2 \cdot 4^{-j} \enspace.
    \end{align*}
    On the other hand, for the case $j=0$ we have that
    \[
        \sum_{i=1}^m \brb{g_0(i) - g_0(m+i)}^2 \le m(b-a)^2 \enspace,
    \]
    thus proving the second point of the statement.
    
    Finally, noticing that for each $j \in \{0,1,\dots,l\}$, the map $\pi_j \colon \cG_{j}(\mathbf{x}) \to \cH_{j}(\mathbf{x})$ is a surjective map, we have that $|\cH_{j}(\mathbf{x})| \le |\cG_{j}(\mathbf{x})|$, which, together with the fact that $\cG_{j}(\mathbf{x})$ is a $(b-a)2^{-j}$-separated set of $\brb{\cF_{\e}(\mathbf{x}), d_2}$, yields the third point of the statement.
\end{proof}

The last ingredient to prove the theorem is the following lemma, which is an immediate corollary of \cite[Corollary~5.4]{Rudelson2006}, (noticing that the metric $d_2$ is the natural metric on $L^2(\mu)$ when the underlying measure $\mu$ is the uniform probability measure on the set $[m]$)

\begin{restatable}{relem}{packingnumberlemma} \label{lem:Vershynin}
    There exist universal constants $\tilde{C}, \tilde{c} > 0$ for which the following holds. For any $\mathbf{x} \in \cX^{2m}$ and any $0 < \zeta < (b-a)/2$ such that $\fat_{\tilde{c} \zeta}(\cF) < \infty$,
    \[
        \cP(\cF(\mathbf{x}), d_2 ,\zeta) \le \lrb{\frac{b-a}{\tilde{c} \zeta}}^{\tilde{C} \fat_{\tilde{c} \zeta}(\cF)} \enspace.
    \]
\end{restatable}
\begin{proof}
Let $\tilde{C},\tilde{c}$ be universal constants as in \cite[Corollary~5.4]{Rudelson2006}.
Define $\cF'(\mathbf{x}) \coloneqq \{(f-a)/(b-a) \mid f \in \cF(\mathbf{x})\}$ and $\zeta' \coloneqq \zeta/(b-a) \in (0,1/2)$.
A direct computation shows that $\fat_{\tilde{c}\zeta'}\brb{\cF'(\mathbf{x})} = \fat_{\tilde{c}\zeta}\brb{\cF(\mathbf{x})}$ and $\cP(\cF'(\mathbf{x}), d_2, \zeta') = \cP(\cF(\mathbf{x}), d_2, \zeta)$.
Now, further observing that $\cF'(\mathbf{x})$ is $1$-bounded in $d_2$ since $\cF(\mathbf{x})+\lcb{-a}$ is $(b-a)$-bounded in $d_2$, we may apply \cite[Corollary~5.4]{Rudelson2006} to $\cF'(\mathbf{x})$ with $\zeta'$ to infer that
\begin{align*}
    \cP(\cF(\mathbf{x}), d_2 ,\zeta)
    &= \cP(\cF'(\mathbf{x}), d_2 ,\zeta') \\
    &\le \lrb{\frac{1}{\tilde{c} \zeta'}}^{\tilde{C} \fat_{\tilde{c} \zeta'}(\cF'(\mathbf{x}))}
    = \lrb{\frac{b-a}{\tilde{c} \zeta}}^{\tilde{C} \fat_{\tilde{c} \zeta}(\cF(\mathbf{x}))} \enspace.
\end{align*}
Finally, we arrive at the conclusion by observing that $\fat_{\tilde{c} \zeta} \brb{\cF(\mathbf{x})} \le \fat_{\tilde{c} \zeta}(\cF)$.
\end{proof}

\section{Proof of \Cref{thm:uniform-convergence}}

We are now ready to present the proof of \Cref{thm:uniform-convergence}.

\begin{proof}
We may assume that $\e<b-a$, since otherwise
\[
    \P \lrb{\sup_{f \in \cF} \labs{\frac{1}{m} \sum_{i=1}^m f(X_i) - \E\lsb{f(X)}} \le \e} = 1 \enspace.
\]
Pick $\tilde{C}$ and $\tilde{c}$ as the universal constants whose existence is stated in \Cref{lem:Vershynin}.
Let $\kappa \coloneqq \fat_{\tilde{c} \e/16}(\cF)$ and $R \coloneqq b-a$.
Furthermore, define $c_j \coloneqq \frac{1}{44}\sqrt{4^{2-j}(j+1)}$ for each $j \in \{0,1,\dots,l\}$.
Then,
{
\allowdisplaybreaks
    \begin{align}
        &\P\lrb{\sup_{f \in \cF} \labs{\frac{1}{m} \sum_{i=1}^m f(X_i) - \E\lsb{f(X)}} > \e} \nonumber\\
        &\overset{\mathrm{\ref{i:uconv-a}}}{\le}
        2 \P\lr*{\sup_{f \in \cF} \abs*{\frac1m \sum_{i=1}^m \lr{f(X_i) - f(X_{m+i})}} > \frac{\e}{2}} \nonumber\\
        & \overset{\mathrm{\ref{i:uconv-b}}}{\le}
        2 \sup_{(x_1, \dots, x_{2m}) \in \cX^{2m}} \P\lr*{\sup_{f \in \cF} \abs*{\frac1m \sum_{i=1}^m Z_i\lr{f(x_i) - f(x_{m+i})}} > \frac{\e}{2}} \nonumber\\
        & \overset{\mathrm{\ref{i:uconv-c}}}{\le}
        2 \sup_{\mathbf{x} \in \cX^{2m}} \P\lr*{\sup_{f \in \cF_\e(\mathbf{x})} \abs*{\frac1m \sum_{i=1}^m Z_i\lr{f(i) - f(m+i)}} > \frac{\e}{4}} \nonumber\\
        &\overset{\mathrm{\ref{i:uconv-d}}}{\le}
        4 \sum_{j=0}^{l} \sup_{\mathbf{x} \in \cX^{2m}} \sum_{h \in \cH_j(\mathbf{x})} \exp\lr*{-\frac{\e^2c_j^2 m^2}{2\sum_{i=1}^m \brb{h(i) - h(m+i)}^2}} \nonumber\\
        &\overset{\mathrm{\ref{i:uconv-e}}}{\le}
        4\sum_{j=0}^l \sup_{\mathbf{x} \in \cX^{2m}} \labs{\cH_j(\mathbf{x})} \exp\lr*{-\frac{1}{32 R^2} \cdot 4^j\e^2 c_j^2 m} \nonumber\\
        &\overset{\mathrm{\ref{i:uconv-f}}}{\le}
        4\sum_{j=0}^l \sup_{\mathbf{x} \in \cX^{2m}}  \cP(\cF(\mathbf{x}), d_2 ,R\cdot2^{-j}) \exp\lr*{-\frac{1}{32 R^2} \cdot 4^j\e^2 c_j^2 m} \nonumber\\
        &\overset{\mathrm{\ref{i:uconv-g}}}{\le}
        4 \sum_{j=0}^l \lrb{\frac{2^{j+2}}{\tilde{c} }}^{\tilde{C} \fat_{\tilde{c} R \cdot 2^{-j}}(\cF)} \exp\lr*{-\frac{1}{32 R^2} \cdot 4^j\e^2 c_j^2 m} \nonumber\\
        &\overset{\mathrm{\ref{i:uconv-h}}}{\le}
        4 \max\lrc*{\frac{4}{\tilde{c}}, 1}^{\tilde{C}\kappa} \sum_{j=0}^l \exp\bbrb{j \cdot \tilde{C} \kappa \ln(2)-\frac{1}{32 R^2} \cdot 4^j\e^2 c_j^2 m} \nonumber\\
        &\overset{\mathrm{\ref{i:uconv-i}}}{=}
        4 \max\lrc*{\frac{4}{\tilde{c}}, 1}^{\tilde{C}\kappa} e^{- \frac{\e^2 m}{2\cdot 44^2 R^2}} \sum_{j=0}^l  \exp\bbrb{j \Brb{\tilde{C} \cdot \kappa \ln(2)  - \frac{\e^2 m}{2\cdot 44^2 R^2} } } \nonumber\\
        &\overset{\mathrm{\ref{i:uconv-j}}}{\le}
        8 \max\lrc*{\frac{4}{\tilde{c}}, 1}^{\tilde{C} \kappa} e^{-\frac{\e^2 m}{2\cdot 44^2 R^2}}
        \label{eq:uniform-convergence-chaining}
    \end{align}
}
    where the marked inequalities respectively follow as explained (in order) by the following points:
    \begin{refenumerate}[itemsep=1pt]
        \item \label{i:uconv-a} By \Cref{lem:symmetrization}, assuming $m \ge 4\ln(2)\cdot (R/\e)^2$. 
        \item \label{i:uconv-b} By \Cref{lem:permutation}, in the light of the fact that $Z_1,\dots,Z_m$ is a family of $\P$-independent Rademacher random variables. 
        \item \label{i:uconv-c} By \Cref{lem:net}.
        \item \label{i:uconv-d} By \Cref{lem:multiscale} with the choice $\tilde{l} \coloneqq l$, for each $j \in \{0,1,\dots,l\}$, $\tilde{\cH}_j(\mathbf{x}) \coloneqq \cH_j(\mathbf{x})$ and $\e_j \coloneqq c_j\e$. Notice that the assumptions of \Cref{lem:multiscale} are satisfied as a consequence of the first point of \Cref{lem:chaining}, and the fact that $\sum_{j=0}^l c_j \le 1/4$. 
        \item \label{i:uconv-e} By the second point in \Cref{lem:chaining}.
        \item \label{i:uconv-f} By the third point in \Cref{lem:chaining} and the fact that $\cP\brb{\cF_\e(\mathbf{x}),d_2,(b-a) \cdot 2^{-j}} \le \cP\brb{\cF(\mathbf{x}),d_2,(b-a) \cdot 2^{-j}}$.
        \item \label{i:uconv-g} By \Cref{lem:Vershynin}. Specifically, if $j \ge 2$, we set $\zeta \coloneqq R \cdot 2^{-j}$ (and upper bound). Instead, if $j \in \{0,1\} $, we first upper bound $\cP\lrb{\cF(\mathrm{x}),d_2,R\cdot2^{-j}}$ with $\cP\lrb{\cF(\mathrm{x}),d_2,R/4}$, then apply the lemma setting $\zeta \coloneqq R/4$ (and upper bound again).
        \item \label{i:uconv-h} By the fact that the function $\gamma \mapsto \fat_{\gamma}(\cF)$ is monotonically non-increasing, and $2^{-l} \ge \e/(16R)$.
        \item \label{i:uconv-i} By our choice of $c_0,\dots,c_l$.
        \item \label{i:uconv-j} Assuming $m \ge 88^2 \ln(2) \cdot R^2 \cdot \e^{-2} \cdot \brb{\tilde{C} \kappa + 1}$.
    \end{refenumerate}

    Finally, observe that the right-hand side of~\eqref{eq:uniform-convergence-chaining} is at most $\delta$ for
    \begin{equation}
    \label{eq:sample-complexity-bound}
        m \ge \frac{2\cdot 44^2 \cdot R^2}{\e^2} \lr*{\tilde{C} \kappa \ln\lr*{8 \max\lrc*{\frac{4}{\tilde{c}}, 1}} + \ln\frac{1}{\delta}} \enspace.
    \end{equation}
    Therefore, for a sufficiently large universal constant $C > 0$ in the statement of the theorem and for $c \coloneqq \tilde{c}/16$, we see that any value of $m$ that satisfies~\eqref{eq:uniform-convergence-sample-complexity} suffices to guarantee~\eqref{eq:sample-complexity-bound} and the assumptions in \cref{i:uconv-a,i:uconv-j}, concluding the proof.
\end{proof}

\bibliographystyle{elsarticle-harv} 
\bibliography{cas-refs}

\end{document}